\RequirePackage{fix-cm}
\documentclass[smallextended]{svjour3}       
\smartqed  
\usepackage{graphicx}
\usepackage[square,numbers]{natbib}
\usepackage{mathtools}
\usepackage{booktabs}
\usepackage{amssymb,bm}
\usepackage[colorlinks,linkcolor=blue,anchorcolor=blue,citecolor=blue]{hyperref}
\usepackage{algorithm}
\usepackage{algorithmic}

\newcommand{\tabincell}[2]{\begin{tabular}{@{}#1@{}}#2\end{tabular}}
\DeclareMathOperator*{\argmin}{arg\,min}
\DeclareMathOperator*{\argmax}{arg\,max}

\newcommand{\bx}{\boldsymbol{x}}
\newcommand{\bw}{\boldsymbol{w}}
\newcommand{\bc}{\boldsymbol{c}}
\newcommand{\bl}{\boldsymbol{l}}
\newcommand{\bI}{\boldsymbol{I}}
\newcommand{\bW}{\boldsymbol{W}}
\newcommand{\by}{\boldsymbol{y}}
\newcommand{\bof}{\boldsymbol{f}}
\newcommand{\bo}{\boldsymbol{o}}
\newcommand{\balpha}{\boldsymbol{\alpha}}
\newcommand{\bbeta}{\boldsymbol{\beta}}
\newcommand{\bphi}{\boldsymbol{\phi}}
\newcommand{\bPhi}{\boldsymbol{\Phi}}
\newcommand{\bpsi}{\boldsymbol{\psi}}
\newcommand{\bPsi}{\boldsymbol{\Psi}}
\newcommand{\bsigma}{\boldsymbol{\sigma}}
\newcommand{\tr}{\mathrm{tr}}
\newcommand{\te}{\mathrm{te}}
\newcommand{\eval}{\mathrm{eval}}
\newcommand{\bR}{\mathbb{R}}
\newcommand{\bE}{\mathbb{E}}
\newcommand{\cX}{\mathcal{X}}
\newcommand{\cY}{\mathcal{Y}}

\newcommand{\cF}{\mathcal{F}}
\newcommand{\cG}{\mathcal{G}}

\newcommand{\cA}{\mathcal{A}}
\newcommand{\cZ}{\mathcal{Z}}
\newcommand{\fR}{\mathfrak{R}}
\begin{document}

\title{A One-step Approach to Covariate Shift Adaptation
}


\author{Tianyi Zhang* \and Ikko Yamane \and Nan Lu \and Masashi Sugiyama
}

\institute{T. Zhang \at
              The University of Tokyo/RIKEN, Kashiwa, Japan \\
              *Corresponding author \\
              \email{zhang@ms.k.u-tokyo.ac.jp}           
           \and
           I. Yaname \at
              Université Paris Dauphine-PSL/RIKEN, Paris, France
           \and
           N. Lu \at
              The University of Tokyo/RIKEN, Kashiwa, Japan
           \and
           M. Sugiyama \at
              RIKEN/ The University of Tokyo, Tokyo, Japan
}

\date{Received: date / Accepted: date}

\maketitle

\begin{abstract}
A default assumption in many machine learning scenarios is that the training and test samples are drawn from the \emph{same} probability distribution.
However, such an assumption is often violated in the real world due to non-stationarity of the environment or bias in sample selection.
In this work, we consider a prevalent setting called \emph{covariate shift},
where the input distribution differs between the training and test stages while the conditional distribution of the output given the input remains unchanged.
Most of the existing methods for covariate shift adaptation are two-step approaches, which first calculate the \emph{importance} weights and then conduct \emph{importance-weighted empirical risk minimization}. In this paper, we propose a novel \emph{one-step approach} that jointly learns the predictive model and the associated weights in one optimization by minimizing an upper bound of the test risk.
We theoretically analyze the proposed method and provide a generalization error bound.
We also empirically demonstrate the effectiveness of the proposed method.

\end{abstract}

\section{Introduction}
\label{sec:1}

When developing algorithms of supervised learning, it is commonly assumed that samples used for training and samples used for testing follow the same probability distribution \citep{bishop1995neural,duda2012pattern,hastie2009elements,scholkopf2001learning,Vapnik1998,wahba1990spline}.
However, this common assumption may not be fulfilled in many real-world applications due to sample selection bias or non-stationarity of environments \citep{huang2007correcting,quionero2009dataset,sugiyama2012machine,zadrozny2004learning}.

Covariate shift, which was first introduced by \citet{shimodaira2000improving}, is a prevalent setting for supervised learning in the wild, where the input distribution is different in the training and test phases but the conditional distribution of the output variable given the input variable remains unchanged. Covariate shift is conceivable in many real-world applications such as brain-computer interfacing \citep{li2010application}, emotion recognition \citep{jirayucharoensak2014eeg}, human activity recognition \citep{hachiya2012importance}, spam filtering \citep{bickel2007dirichlet}, or speaker identification \citep{yamada2010semi}.

Due to the difference between the training and test distributions, the model trained by employing standard machine learning techniques such as \emph{empirical risk minimization} \citep{scholkopf2001learning,Vapnik1998} may not generalize well to the test data.
However, as shown by \citet{shimodaira2000improving}, \citet{sugiyama2005input}, \citet{sugiyama2007covariate}, and \citet{zadrozny2004learning}, this problem can be mitigated by \emph{importance sampling} \citep{cochran2007sampling,fishman2013monte,kahn1953methods}: weighting the training loss terms according to the \emph{importance}, which is the ratio of the test and training input densities.
As a consequence, most previous work \citep{huang2007correcting,kanamori2009least,sugiyama2008direct} mainly focused on accurately estimating the importance.
Then the estimated importance is used to train a predictive model in the training phase.
Thus, most of the existing methods of covariate shift adaptation are two-step approaches.

However, according to \emph{Vapnik's principle} \citep{Vapnik1998}, which advocates that one should avoid solving a more general problem as an intermediate step when the amount of information is limited, directly solving the covariate shift problem may be preferable to two-step approaches when the amount of covariate shift is
substantial and the number of training data is not large.
Moreover, \citet{yamada2011relative} argued that density ratio estimation, the intermediate step for covariate shift adaptation, is indeed rather hard,
suggesting that the importance approximation could be unreliable and thus deteriorate the performance of learning in practice.

In this paper, we propose a novel one-step approach to covariate shift adaptation, without the intermediate step of estimating the ratio of the training and test input densities.
We jointly learn the predictive model and the associated weights by minimizing an upper bound of the test risk.
Furthermore, we establish a generalization error bound based on the Rademacher complexity to give a theoretical guarantee for the proposed method.
Experiments on synthetic and benchmark datasets highlight the advantage of our method over the existing two-step approaches.

\section{Preliminaries}
\label{sec:2}

In this section, we briefly introduce the problem setup of covariate shift adaptation and relevant previous methods.

\subsection{Problem Formulation}

Let us start from the setup of supervised learning.
Let $\cX\subset\bR^d$ be the input space ($d$ is a positive integer),
$\cY\subset\bR$ (regression) or $\cY=\{-1,+1\}$ (binary classification) be the output space,
and $\left\{\left(\bx_i^\tr,  y_i^\tr\right)\right\}_{i=1}^{n_\tr}$ be the training samples drawn independently from a training distribution with density $p_\tr(\bx,y)$,
which can be decomposed into the marginal distribution and the conditional probability distribution, i.e., \[p_\tr(\bx,y)=p_\tr(\bx)p_\tr(y|\bx).\]
Let $(\bx^\te, y^\te)$ be a test sample drawn from a test distribution with density \[p_\te(\bx,y)=p_\te(\bx)p_\te(y|\bx).\]

Formally, the goal of supervised learning is to obtain a model $f\colon\cX\rightarrow\bR$ with the training samples that minimizes the expected loss over the test distribution (which is also called the \emph{test risk}):
\begin{align}
    \label{eq: risk}
    R(f) \coloneqq \bE_{(\bx^\te, y^\te)\sim p_\te(\bx, y)}\left[\ell(f(\bx^\te), y^\te)\right],
\end{align}
where $\ell\colon \bR\times\cY \to \bR_+$ denotes the \emph{loss function} that measures the discrepancy between the true output value $y$ and the predicted value $\widehat{y}\coloneqq f(\bx)$.
In this paper, we assume that $\ell$ is bounded from above. We will discuss the practical choice of loss functions in Section~\ref{sec:3}.

\begin{sloppypar}
Since the assumption that the joint distributions are unchanged (i.e., $p_\tr(\bx,y)=p_\te(\bx,y)$) does not hold under covariate shift (i.e., $p_\tr(\bx)\neq p_\te(\bx)$, $\operatorname{supp}(p_\tr)=\operatorname{supp}(p_\te)$, and $p_\tr(y|\bx)=p_\te(y|\bx)$), we utilize unlabeled test samples $\{\bx_i^\te\}_{i=1}^{n_\te}$, which are independently drawn from a distribution with density $p_{\te}(\bx)$, besides the labeled training samples $\left\{\left(\bx_i^\tr,  y_i^\tr\right)\right\}_{i=1}^{n_\tr}$ to compensate the difference of distributions.
The goal of covariate shift adaptation is still to obtain a model that minimizes the test risk~\eqref{eq: risk}.
\end{sloppypar}

\subsection{Previous Work}
\label{sec:2.2}

Empirical risk minimization (ERM) \citep{scholkopf2001learning,Vapnik1998}, a standard technique in supervised learning, may fail under covariate shift due to the difference between the training and test distributions. 

Importance sampling was used to mitigate the influence of covariate shift \citep{shimodaira2000improving,sugiyama2005input,sugiyama2007covariate,zadrozny2004learning}:
\begin{align*}
    \bE_{(\bx^\te, y^\te)\sim p_\te(\bx, y)}\left[\ell(f(\bx^\te), y^\te)\right]
    = \bE_{(\bx^\tr, y^\tr)\sim p_\tr(\bx, y)}\left[\ell(f(\bx^\tr), y^\tr)r(\bx^\tr)\right],
\end{align*}
where \[r(\bx)=p_\te(\bx)/p_\tr(\bx)\] is referred to as the importance, and this leads to the \emph{importance weighted ERM} (IWERM):
\[\min_{f\in\cF}\frac{1}{n_\tr}\sum_{i=1}^{n_\tr}\ell(f(\bx^\tr_i), y^\tr_i)r(\bx^\tr_i),\]
where $\cF$ is a hypothesis set.
For any fixed $f\in\cF$, the importance weighted empirical risk is an unbiased estimator of the test risk.

\begin{sloppypar}
However, IWERM tends to produce an estimator with high variance making the resulting test risk large \citep{shimodaira2000improving,sugiyama2012machine}.
Reducing the variance by slightly flattening the importance weights is practically useful,
which results in \emph{exponentially-flattened importance weighted ERM} (EIWERM) proposed by \citet{shimodaira2000improving}:  \[\min_{f\in\cF}\frac{1}{n_\tr}\sum_{i=1}^{n_\tr}\ell(f(\bx^\tr_i), y^\tr_i)r(\bx^\tr_i)^\gamma,\]
where $\gamma\in[0,1]$ is called the flattening parameter.
\end{sloppypar}

Therefore, how to estimate the importance accurately becomes the key to success of covariate shift adaptation.
Unconstrained Least-Squares Importance Fitting (uLSIF) \citep{kanamori2009least} is one of the commonly used density ratio estimation methods which is computationally efficient and comparable to other methods \citep{huang2007correcting,sugiyama2008direct} in terms of performance. It minimizes the following squared error to obtain an importance estimator $\widehat{r}(\bx)$:
\[\bE_{\bx^\tr\sim p_\tr(\bx)}\left[\left(\widehat{r}(\bx^\tr)-r(\bx^\tr)\right)^2\right].\]

\citet{yamada2011relative} argued that estimation of the density ratio is rather hard, which weakens the effectiveness of EIWERM.
Instead, they proposed a method that directly estimates a flattened version of the importance weights, called \emph{relative importance weighted ERM} (RIWERM):
\[\min_{f\in\cF}\frac{1}{n_\tr}\sum_{i=1}^{n_\tr}\ell(f(\bx^\tr_i), y^\tr_i)r_\alpha(\bx^\tr_i),\]
where \[r_\alpha(\bx)\coloneqq \frac{p_\te(\bx)}{\alpha p_\te(\bx)+(1-\alpha) p_\tr(\bx)}\] is called the $\alpha$-relative importance $(\alpha\in[0,1])$. The relative importance $r_\alpha(\bx)$ can be estimated by relative uLSIF (RuLSIF) as presented by \citet{yamada2011relative}, which minimizes the following squared error to obtain a relative importance estimator $\widehat{r}_\alpha(\bx)$:
\[\bE_{\bx'\sim \alpha p_\te(\bx)+(1-\alpha) p_\tr(\bx)}\left[\left(\widehat{r}_\alpha(\bx')-r_\alpha(\bx')\right)^2\right].\]

Hyper-parameters such as the flattening parameter $\gamma$ or $\alpha$ need to be appropriately chosen in order to obtain a good generalization capability.
However, cross validation (CV), a standard technique for model selection, does not work well under covariate shift.
To cope with this problem, a variant of CV called importance-weighted CV (IWCV) has been proposed by \citet{sugiyama2007covariate}, which is based on the importance sampling technique to give an almost unbiased estimate of the generalization error with finite samples.
However, the importance used in IWCV still needs to be estimated from samples.

As reviewed above, existing methods of covariate shift adaptations are two-step approaches: they first estimate the weights (importance or its variant), and then employ the weighted version of ERM for training a prediction model. However, these methods introduce a more general problem as an intermediate step, which violates Vapnik's principle and may be 
sub-optimal.

\section{Proposed Method}
\label{sec:3}

In this section, in order to overcome the drawbacks of the existing two-step approaches, we propose a one-step approach which integrates the importance estimation step and the importance-weighted empirical risk minimization step by upper-bounding the test risk.
Moreover, we provide a theoretical analysis of the proposed method.

\subsection{One-step Approach}
\label{sec:3.1}

First, we derive an upper bound of the test risk, which is the key of our one-step approach.

\begin{theorem}
\label{thm: upper bound}
Let $r(\bx)$ be the importance $p_\te(\bx)/p_\tr(\bx)$ and $\cF\subseteq \{f\colon\cX\to\bR\}$ be a given hypothesis set.
Suppose that there is a constant $m \in \bR$ such that $\ell(f(\bx), y)\leq m$ holds for every $f\in\cF$ and every $(\bx, y)\in\cX\times\cY$.
Then, for any $f\in\cF$ and any measurable function $g\colon\cX\to\bR$, the test risk is bounded as follows under covariate shift:
\begin{align}
    \label{eq: upper bound1}
    \frac{1}{2} R^2(f) \le J(f, g) &\coloneqq \left(\bE_{(\bx^\tr, y^\tr)\sim p_\tr(\bx, y)}\left[\ell(f(\bx^\tr), y^\tr)g(\bx^\tr)\right]\right)^2\notag\\
    &\phantom{\coloneqq}\ + m^2\bE_{\bx^\tr\sim p_\tr(\bx)}\left[\left(g(\bx^\tr)-r(\bx^\tr)\right)^2\right].
\end{align}
Furthermore, if $g$ is non-negative and $\ell_\mathrm{UB}$ bounds $\ell$ from above,
we have
\begin{align}
    \label{eq: upper bound2}
    J(f, g) \le J_\mathrm{UB}(f, g)
    &\coloneqq \left(\bE_{(\bx^\tr, y^\tr)\sim p_\tr(\bx, y)}\left[\ell_\mathrm{UB}(f(\bx^\tr), y^\tr)g(\bx^\tr)\right]\right)^2\notag\\
    &\phantom{\coloneqq}\ + m^2\bE_{\bx^\tr\sim p_\tr(\bx)}\left[\left(g(\bx^\tr)-r(\bx^\tr)\right)^2\right].
\end{align}
\end{theorem}

\begin{proof}
According to the Cauchy-Schwarz inequality, we have
\begin{align*}
    \frac{1}{2}R^2(f)
    &=\frac{1}{2}\left(\bE_{(\bx^\tr, y^\tr)}\left[\ell(f(\bx^\tr), y^\tr)r(\bx^\tr)\right]\right)^2\\
    &\leq\left(\bE_{(\bx^\tr, y^\tr)}\left[\ell(f(\bx^\tr), y^\tr)g(\bx^\tr)\right]\right)^2\\
    &\phantom{\leq}\ 
    +\left(\bE_{(\bx^\tr, y^\tr)}\left[\ell(f(\bx^\tr), y^\tr)\left(r(\bx^\tr)-g(\bx^\tr)\right)\right]\right)^2\\
    &\leq\left(\bE_{(\bx^\tr, y^\tr)}\left[\ell(f(\bx^\tr), y^\tr)g(\bx^\tr)\right]\right)^2\\
    &\phantom{\leq}\ 
    +\bE_{(\bx^\tr, y^\tr)}\left[\ell^2(f(\bx^\tr), y^\tr)\right]\bE_{\bx^\tr}\left[\left(g(\bx^\tr)-r(\bx^\tr)\right)^2\right]\\
    &\leq\left(\bE_{(\bx^\tr, y^\tr)}\left[\ell(f(\bx^\tr), y^\tr)g(\bx^\tr)\right]\right)^2
    +m^2\bE_{\bx^\tr}\left[\left(g(\bx^\tr)-r(\bx^\tr)\right)^2\right],
\end{align*}
where $(\bx^\tr, y^\tr) \sim p(\bx^\tr, y^\tr)$. 
This proves \eqref{eq: upper bound1}, and based on this, \eqref{eq: upper bound2} is obvious.
\end{proof}

For classification problems, $R(f)$ is typically defined by the zero-one loss $\ell(\widehat{y}, y)=I(\widehat{y}y\leq0)$, where $I$ is the indicator function, and thus the boundedness assumption of the loss function in Theorem~\ref{thm: upper bound} holds with $m=1$.
For regression problems, The squared loss $\ell(\widehat{y}, y)=(\widehat{y}-y)^2$ is a typical choice, but it violates the boundedness assumption.
Instead, we define $R(f)$ using Tukey's bisquare loss \citep{beaton1974fitting} (see Fig.~\ref{fig:tukey}).
\footnote{There is another bounded loss called the Welsch loss \citep{8954089,doi:10.1080/03610917808812083,9020208} which has a similar shape to that of Tukey's bisquare loss. In this paper, we focus on Tukey's bisquare loss.}

\begin{figure*}[bt]
  \centering
  \includegraphics[scale=0.5]{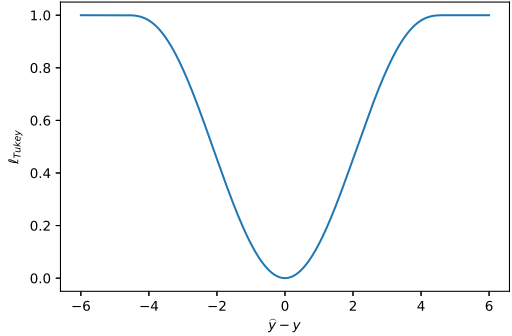}
  \caption{Tukey's loss defined as 
  $\ell_{\rm Tukey}(\widehat{y}, y) \coloneqq \min\left(1-\left[1-(\widehat{y}-y)^2/\rho^2\right]^{3}, 1\right) \le 1$.
  It has been widely used in the context of robust statistics.
  The hyper-parameter $\rho > 0$ is usually set to $4.685$ for this loss function,
  and it provides an asymptotic efficiency $95\%$ of that of least squares for Gaussian noise \citep{andersen2008modern}.
  Here, we rescale the standard Tukey's bisquare loss for convenience, which does not change the minimizer of the test risk.}
  \label{fig:tukey}
\end{figure*}

\begin{remark}
The two-step approach that first applies uLSIF to estimate the importance weights and then employs IWERM is equivalent to minimizing the second term of the above upper bounds first and then minimizing the first term, which leads to a sub-optimal solution to the upper-bound minimization. 
\end{remark}

Instead of estimating the unknown $r(\bx)$ for minimizing $R(f)$ as in the previous two-step approaches,
we propose a one-step approach that minimizes the upper bound $J(f, g)$ or $J_\mathrm{UB}(f, g)$ based on Theorem~\ref{thm: upper bound}.

For classification problems, $J(f, g)$ is defined using the zero-one loss, with which training will not be tractable \citep{bendavid06jcss}.
Fortunately, the latter part of Theorem~\ref{thm: upper bound} allows us to minimize $J_\mathrm{UB}(f, g)$ instead,
with $\ell_\mathrm{UB}$ being any (sub-)differentiable approximation that bounds the zero-one loss from above
so that we can apply any gradient method
such as stochastic gradient descent \citep{robbins1951stochastic}.
Examples of such $\ell_\mathrm{UB}$ include the hinge loss $\ell(\widehat{y}, y)=\operatorname{max}(0, 1-\widehat{y}y)$ and the squared loss.
For regression problems, Tukey's loss is already differentiable,
but we can use the squared loss that bounds Tukey's loss which makes the optimization problem simpler as described later.
This is again justified by Theorem~\ref{thm: upper bound} with the squared loss used for the upper-bound loss $\ell_\mathrm{UB}$.

Although the second expectation in $J_\mathrm{UB}(f, g)$ contains an unknown term $r(\bx)$, it can be estimated from the samples on hand, up to addition by a constant due to the fact that
\begin{align*}
    &\phantom{=\ }\bE_{\bx^\tr\sim p_\tr(\bx)}\left[\left(g(\bx^\tr)-r(\bx^\tr)\right)^2\right]\\
    &=\bE_{\bx^\tr\sim p_\tr(\bx)}\left[g^2(\bx^\tr)\right]-2\bE_{\bx^\te\sim p_\te(\bx)}\left[g(\bx^\te)\right]+C,
\end{align*}
where $C$ is a constant that does not depend on the function $f$ nor $g$.

Since the true distributions are unknown, we minimize its empirical version $\widehat{J}_\mathrm{UB}(f,g;S)$ with respect to $f$ and non-negative $g$ in some given hypothesis sets $\cF$ and $\cG_+$:
\begin{align}
    \label{eq: empirical}
    \widehat{J}_\mathrm{UB}(f, g; S)
    &\coloneqq \left(\frac{1}{n_\tr}\sum_{i=1}^{n_\tr}\ell_\mathrm{UB}(f(\bx^\tr_i), y^\tr_i)g(\bx^\tr_i)\right)^2\notag\\
    &\phantom{\coloneqq} +m^2\left(\frac{1}{n_\tr}\sum_{i=1}^{n_\tr}g^2(\bx^\tr_i)-\frac{2}{n_\te}\sum_{i=1}^{n_\te}g(\bx^\te_i)+C\right),
\end{align}
where $S\coloneqq\left\{\left(\bx_i^\tr, y_i^\tr\right)\right\}_{i=1}^{n_\tr}\cup\left\{\bx_i^\te\right\}_{i=1}^{n_\te}$ is the set of sample points.
Notice that constant $C$ can be safely ignored in the minimization.

Below, we present an alternating minimization algorithm described in Algorithm~\ref{alg:alternating}
that can be employed when $f(\bx)$ and $g(\bx)$ are linear-in-parameter models, i.e.,
\begin{equation}
    \label{linear model}
    f(\bx)=\balpha^\top \bphi(\bx)
    \quad \text{and}\quad 
    g(\bx)=\bbeta^\top \bpsi(\bx),
\end{equation}
where $\balpha\in\bR^{b_f}$ and $\bbeta\in\bR^{b_g}$ are parameters, and $\bphi$ and $\bpsi$ are $b_f$-dimensional and $b_g$-dimensional vectors of basis functions.

\begin{sloppypar}
First, we minimize the objective~\eqref{eq: empirical} with respect to $g$ while fixing $f$.
This step has an analytic solution as shown in Algorithm~\ref{alg:alternating}, Line~6,
where $\bPhi_\tr=(\bphi(\bx^\tr_1),\ldots,\bphi(\bx^\tr_{n_\tr}))^\top$,
$\bPsi_\tr=(\bpsi(\bx^\tr_1),\ldots,\bpsi(\bx^\tr_{n_\tr}))^\top$,
$\bPsi_\te = (\bpsi(\bx^\te_1),\ldots,\bpsi(\bx^\te_{n_\te}))^\top$,
$\boldsymbol{1}=(1,\ldots,1)^\top$,
and $\bI$ is the identity matrix.
\end{sloppypar}

Next, we minimize the objective~\eqref{eq: empirical} with respect to $f$ while fixing $g$.
In this step, we can safely ignore the second term and remove the square operation of the first term in the objective~\eqref{eq: empirical}
to reduce the problem into weighted empirical risk minimization (cf.\@ Algorithm~\ref{alg:alternating}, Line~12)
by forcing $g$ to be non-negative with a rounding up technique \citep{kanamori2009least} as shown in Algorithm~\ref{alg:alternating}, Line~7.
For regression problems, the method of iteratively reweighted least squares (IRLS) \citep{beaton1974fitting} can be used for optimizing Tukey's bisquare loss.
In practice, we suggest using the squared loss as a convex approximation of Tukey's loss to obtain a closed-form solution as shown in Algorithm~\ref{alg:alternating}, Line~10 for reducing computation time, and we compare their performance in the experiments.
For classification with linear-in-parameter models using the hinge loss, then the weighted support vector machine \citep{yang2007weighted} can be used.
After this step, we go back to the step for updating $g$ and repeat the procedure.

\begin{algorithm}[ht]
\caption{Alternating Minimization with Linear-in-parameter Models}
\label{alg:alternating}
\begin{algorithmic}[1]
\linespread{1.1}\selectfont
\STATE $\balpha_0\gets$ an arbitrary $b_f$-dimensional vector
\STATE $\lambda_g\gets$ a positive $\ell_2$-regularization parameter
\STATE $\lambda_f\gets$ a positive $\ell_2$-regularization parameter
\FOR{$t=0, 1, \ldots, T-1$}
\STATE $\bl_t\gets(\ell_\mathrm{UB}(\balpha_t^\top\bphi(\bx^\tr_1),y^\tr_1),\ldots,\ell_\mathrm{UB}(\balpha_t^\top\bphi(\bx^\tr_{n_\tr}), y^\tr_{n_\tr}))^\top$
\STATE $\bbeta_{t+1}\leftarrow\left(\frac{1}{n_\tr}\bPsi_\tr^\top\bPsi_\tr+\frac{1}{m^2 n_\tr^2}\bPsi_\tr^\top\bl_t\bl_t^\top\bPsi_\tr+\frac{1}{m^2}\lambda_g\bI\right)^{-1}\frac{1}{n_\te}\bPsi_\te^\top\boldsymbol{1}$
\STATE $\bbeta_{t+1}\leftarrow\max(\bbeta_{t+1},\boldsymbol{0})$\
\STATE $w_i^{t+1} \gets \bbeta_{t+1}^\top\bpsi(\bx_i^\tr)$, $i=1,\ldots,n_\tr$
\IF{$\ell_\mathrm{UB}$ is the squared loss}
  \STATE $\balpha_{t+1}\leftarrow\left(\bPhi_\tr^\top\bW_{t+1}\bPhi_\tr+\lambda_f n_\tr\bI\right)^{-1}\bPhi_\tr^\top\bW_{t+1}\by_\tr$,\\
  where $\bW_{t+1}=\operatorname{diag}(w_1^{t+1},\ldots,w_{n_\tr}^{t+1})$ and $\by_\tr=(y_1^\tr,\ldots,y_{n_\tr}^\tr)^\top$
\ELSE
  \STATE $\balpha_{t+1}\leftarrow\argmin_{\balpha}\frac{1}{n_\tr}\sum_{i=1}^{n_\tr}w_i^{t+1}\ell_\mathrm{UB}(\balpha_t^\top\bphi(\bx^\tr_i),y^\tr_i)+\lambda_f\balpha^\top\balpha$
\ENDIF
\ENDFOR
\end{algorithmic}
\end{algorithm}

\subsection{Theoretical Analysis}
In what follows, we establish a generalization error bound for the proposed method in terms of the \emph{Rademacher complexity} \citep{koltchinskii_rademacher_2001}.

\begin{lemma}
\label{lemma: uniform bound}
Assume that
(a) there exist some constants $M\ge m$ and $L>0$ such that $\ell_\mathrm{UB}(f(\bx), y)\leq M$ holds for every $f\in\cF$ and every $(\bx, y)\in\cX\times\cY$ and $y\mapsto\ell_\mathrm{UB}(y, y')$ is $L$-Lipschitz for every fixed $y'\in\cY$;\footnote{This assumption is valid when $\sup _{f \in \cF}\|f\|_{\infty}$ and $\sup_{y\in\cY}|y|$ are bounded.}
(b) there exists some constant $G\geq 1$ such that $g(\bx)\leq G$ for every $g\in\cG_+$ and every $\bx\in\cX$. Let $\cG=\cG_+\cup-\cG_+$
Then for any $\delta>0$, with probability at least $1-\delta$ over the draw of $S$, the following holds for all $f\in\cF, g\in\cG_+$ uniformly:
\begin{align}
    \label{eq: uniform bound}
    J_\mathrm{UB}(f, g) \leq &\widehat{J}_\mathrm{UB}(f, g; S) + 8MG\left(M+G\right)\left(L\fR^\tr_{n_\tr}(\cF)+\fR^\tr_{n_\tr}(\cG)\right)\notag\\
    &+ 4M^2\fR^\te_{n_\te}(\cG) + 5M^2G^2\sqrt{\frac{\log\frac{1}{\delta}}{2}}\left(\frac{1}{\sqrt{n_\tr}}+\frac{1}{\sqrt{n_\te}}\right),
\end{align}
where $\fR^\tr_{n_\tr}(\cF)$ and $\fR^\tr_{n_\tr}(\cG)$ are the Rademacher complexities of $\cF$ and $\cG$, respectively, for the sampling of size $n_\tr$ from $p_\tr(\bx)$, and $\fR^\te_{n_\te}(\cG)$ is the Rademacher complexity of $\cG$ for the sampling of size $n_\te$ from $p_\te(\bx)$.
\end{lemma}

We provide a proof of Lemma~\ref{lemma: uniform bound} in Appendix~\ref{sec: proof of uniform bound lemma}.
Combining \eqref{eq: upper bound1}, \eqref{eq: upper bound2}, and \eqref{eq: uniform bound}, we obtain the following theorem.

\begin{theorem}
\label{thm: uniform generalization bound}
Suppose that the assumptions in Lemma~\ref{lemma: uniform bound} hold.
Then, for any $\delta>0$, with probability at least $1-\delta$ over the draw of $S$, the test risk can be bounded as follows for all $f\in\cF$ uniformly:
\begin{align}
    \frac{1}{2}R^2(f) \leq &\min_{g\in\cG_+}\widehat{J}_\mathrm{UB}(f, g; S) + 8MG\left(M+G\right)\left(L\fR^\tr_{n_\tr}(\cF)+\fR^\tr_{n_\tr}(\cG)\right)\notag\\
    &+ 4M^2\fR^\te_{n_\te}(\cG) + 5M^2G^2\sqrt{\frac{\log\frac{1}{\delta}}{2}}\left(\frac{1}{\sqrt{n_\tr}}+\frac{1}{\sqrt{n_\te}}\right).
\end{align}
\end{theorem}

Theorem~\ref{thm: uniform generalization bound} implies that minimizing $\widehat{J}_\mathrm{UB}(f, g; S)$, as the proposed method does, amounts to minimizing an upper bound of the test risk.
Furthermore, the following theorem shows a generalization error bound for the minimizer obtained by the proposed method.

\begin{theorem}
\label{thm: generalization error bound}
Let $(\widehat{f}, \widehat{g}) = \argmin_{(f,g)\in\cF\times\cG_+}\widehat{J}_\mathrm{UB}(f, g; S)$.
Then, under the assumptions of Lemma~\ref{lemma: uniform bound}, for any $\delta>0$, it holds with probability at least $1-\delta$ over the draw of $S$ that
\begin{align}
    \frac{1}{2}R^2(\widehat{f}) \leq &\min_{f\in\cF, g\in\cG_+}J_\mathrm{UB}(f, g) + 8MG\left(M+G\right)\left(L\fR^\tr_{n_\tr}(\cF)+\fR^\tr_{n_\tr}(\cG)\right) \notag\\
    &+ 4M^2\fR^\te_{n_\te}(\cG)+ 10M^2G^2\sqrt{\frac{\log\frac{1}{\delta}}{2}}\left(\frac{1}{\sqrt{n_\tr}}+\frac{1}{\sqrt{n_\te}}\right) + \frac{M^2G^2}{n_\tr}. \label{eq: gen error bound}
\end{align}
\end{theorem}

\begin{sloppypar}
A proof of Theorem~\ref{thm: generalization error bound} is presented in Appendix~\ref{sec: proof of generalization error bound theorem}.
If we use linear-in-parameter models with bounded norms, then $\fR^\tr_{n_\tr}(\cF)=O(1/\sqrt{n_\tr})$, $\fR^\tr_{n_\tr}(\cG)=O(1/\sqrt{n_\tr})$, and $\fR^\te_{n_\te}(\cG)=O(1/\sqrt{n_\te})$ \citep{mohri2018foundations,shalev2014understanding}. Furthermore, if we assume that the approximation error of $\cG_+$ is zero, i.e., $r\in\cG_+$, then $\min_{f\in\cF, g\in\cG+}J_\mathrm{UB}(f, g)\leq J_\mathrm{UB}(f^*,r)=R_\mathrm{UB}^2(f^*)$, where $R_\mathrm{UB}$ is the test risk defined with $\ell_\mathrm{UB}$ and $f^*=\argmin_{f\in\cF}R_\mathrm{UB}(f)$.
Thus,
\begin{equation*}
    R(\widehat{f})\leq\sqrt{2}R_\mathrm{UB}(f^*)+O_p(1/\sqrt[4]{n_\tr}+1/\sqrt[4]{n_\te}).
\end{equation*}
When the best-in-class test risk $R_\mathrm{UB}(f^*)$ is small, this bound would theoretically guarantee a good performance of the proposed method.
\end{sloppypar}

\section{Experiments on Regression and Binary Classification}
\label{sec:4}
In this section, we examine the effectiveness of the proposed method via experiments on toy regression and binary classification benchmark datasets.

\subsection{Illustration with Toy Regression Datasets}
\label{sec:4.1}

First, we conduct experiments on a toy regression dataset. 

Let us consider a one-dimensional regression problem.
Let the training and test input densities be 
\[p_{\mathrm{tr}}(x)=N(x ; 1,(0.5)^{2}) \text{\quad and \quad  } p_{\mathrm{te}}(x)=N(x ; 2,(0.25)^{2}),\]
where $N(x ; \mu, \sigma^{2})$ denotes the Gaussian density with mean $\mu$ and variance $\sigma^2$.
Consider the case where the output labels of examples are generated by
\begin{equation*}
    y=f^*(x)+\epsilon \quad \text{with} \quad f^*(x)=\operatorname{sinc}(x),
\end{equation*}
and the noise $\epsilon$ following $N\left(0,(0.1)^{2}\right)$ is independent of $x$.
As illustrated in Fig.~\ref{fig:toy}, the training input points are distributed on the left-hand side of the input domain and the test input points are distributed on the right-hand side.
We sample $n_{\tr}=150$ labeled i.i.d.\@ training samples $\left\{\left(x_i^\tr,  y_i^\tr\right)\right\}_{i=1}^{n_\tr}$ with each $x_i^\tr$ following $p_\tr(x)$ and $n_{\te}=150$ unlabeled i.i.d.\@ test samples $\{x_i^\te\}_{i=1}^{n_\te}$ following $p_\te(x)$ for learning the target function $f^*(x)$ in the experiment.
In addition, we sample 10000 labeled i.i.d.\@ test samples $\left\{(x_i^\eval,  y_i^\eval)\right\}_{i=1}^{n_\eval}$ with each $(x_i^\eval,y_i^\eval)$ following $p_\te(x,y)$ for evaluating the performance of learned functions.

\begin{figure*}[bt]
  \centering
  \includegraphics[scale=0.5]{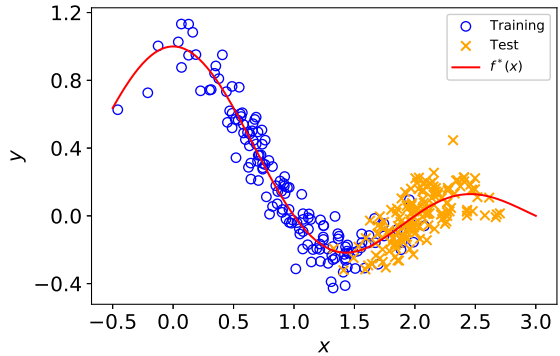}
  \caption{A toy regression example.
  The training input points (blue) are distributed on the left-hand side of the input domain
  and the test input points (orange) are distributed on the right-hand side.
  The two distributions share the same regression function $f^*$ (the red curve).}
  \label{fig:toy}
\end{figure*}

We compare our one-step approach with three baseline methods, which are the ordinary ERM, EIWERM with uLSIF, and RIWERM.
We use the linear-in-parameter models~\eqref{linear model} with the following Gaussian kernels as basis functions for learning the input-output relation and the importance (or the relative importance) in all the experiments including those in Section~\ref{sec:4.2}:
\[\phi_i(\bx)=\exp\bigg\{-\frac{\|\bx-\bc^f_i\|^2}{2\sigma_f^2}\bigg\}\quad\text{and}\quad \psi_i(\bx)=\exp\bigg\{-\frac{\|\bx-\bc^g_i\|^2}{2\sigma_g^2}\bigg\},\]
where $\sigma_f$ and $\sigma_g$ are the bandwidths of the Gaussian kernels, and $\bc_i^f$ and $\bc_i^g$ are the kernel centers randomly chosen from $\{\bx_i^\te\}_{i=1}^{n_\te}$ \citep{kanamori2009least,sugiyama2008direct}.
We set $b_f=b_g=50$ in all the experiments.
Moreover, we use $\ell_2$ regularization in all the experiments, which introduces two more hyperparameters $\lambda_f$ and $\lambda_g$ associated with models $f$ and $g$ respectively.

Let us clarify the hyperparameter tuning procedure for each method.
For the ordinary ERM, the standard cross validation is applied for tuning $\sigma_f$ and $\lambda_f$.
For the EIWERM with uLSIF, the hyperparameter tuning of $\sigma_g$ and $\lambda_g$ in the importance estimation step uses the cross validation naturally based on its learning objective (cf.\@
\citet{kanamori2009least}), and we apply IWCV in the training step for selecting $\sigma_f$, $\lambda_f$ and flattening parameter $\gamma$.
For RIWERM, the built-in cross validation with its learning objective is used for tuning $\sigma_g$ and $\lambda_g$ (cf.\@ \citet{yamada2011relative}), and the selection of $\sigma_f$, $\lambda_f$ and parameter $\alpha$ in the training step is achieved by IWCV (the importance is obtained by uLSIF).
Finally, for our one-step approach, we can naturally perform cross validation based on the proposed learning objective~\eqref{eq: empirical}.
Since all the hyperparameters are tuned simultaneously in the one-step approach, it is computationally expensive as shown in Table~\ref{tab:toy}. 
To make the one-step approach computationally more efficient, we suggest two heuristics to predetermine some of the hyperparameters. 
The first one is to set the bandwidths to the median distances between samples and kernel centers, which is a popular heuristic in practice \citep{scholkopf2001learning}, and we tune the regularization parameters by cross validation. 
For a fair comparison, we also report the results of the baseline methods using the median heuristic. 
The second one is to use the hyperparameters of model $g$ determined by uLSIF and tune the hyperparameters of model $f$ by cross validation, since the one-step approach has a close relationship with uLSIF.

As suggested in Section~\ref{sec:3.1}, we use the squared loss in the one-step approach for computational efficiency. 
We also employ the IRLS algorithm for optimizing Tukey's bisquare loss in the one-step approach.
For better comparison, we report the results of the baseline methods using both the squared loss and Tukey's bisquare loss.

The experimental results of the toy regression problem are summarized in Table~\ref{tab:toy}.
Note that when the target function $f^*$ is perfectly learned, the mean squared error is the variance of $\epsilon$, i.e., 0.01.
Therefore, our method significantly mitigates the influence of covariate shift. Since the IRLS algorithm is needed when using Tukey's bisquare loss, the training should take longer time than that when using the squared loss, and we confirm it according to the results in Table~\ref{tab:toy}.
\begin{table}[t]
\centering
\caption{Mean squared test errors averaged over 100 trials on the toy dataset.
The numbers in the brackets are the standard deviations.
The best method and comparable ones based on the \emph{paired t-test} at the significance level 5\% are described in bold face.
The computation time is averaged over 100 trials.
``Squared'' denotes the squared loss, ``Tukey'' denotes Tukey's bisquare loss, ``median'' means that the bandwidths of the kernel models are determined by the median heuristic (other hyperparameters are still chosen by cross validation), and ``uLSIF'' means that the hyperparameters of model $g$ are the same as those used in uLSIF (the hyperparameters of model $f$ are still chosen by cross validation).}
\vspace{2mm}
\label{tab:toy}
\begin{tabular}{l|c|c}
\toprule[1pt]
Methods & MSE(SD) & Computation time (sec) \\
\midrule
ERM (squared) & 0.0517 (0.0300) & 0.25 \\
EIWERM (squared) & 0.0265 (0.0361) & 2.71 \\
RIWERM (squared) & 0.0265 (0.0405) & 3.90 \\
one-step (squared) & 0.0173 (0.0107) & 64.68 \\
one-step (squared, uLSIF) & 0.0173 (0.0107) & 0.87 \\
ERM (Tukey) & 0.0511 (0.0455) & 0.62\\
EIWERM (Tukey) & 0.0248 (0.0430) & 6.03 \\
RIWERM (Tukey) & 0.0227 (0.0223) & 7.17 \\
one-step (Tukey) & 0.0163 (0.0093) & 134.16 \\
one-step (Tukey, uLSIF) & 0.0163 (0.0093) & 1.66 \\
ERM (squared, median) & 0.1453 (0.1812) & 0.04 \\
EIWERM (squared, median) & 0.0198 (0.0151) & 0.33 \\
RIWERM (squared, median) & 0.0162 (0.0100) & 0.47 \\
one-step (squared, median) & 0.0131 (0.0036) & 0.73 \\
ERM (Tukey, median)   & 0.0760 (0.0733) & 0.09 \\
EIWERM (Tukey, median) & 0.0161 (0.0106) & 0.76 \\
RIWERM (Tukey, median) & 0.0149 (0.0073) & 0.84\\
one-step (Tukey, median) & \textbf{0.0125} (\textbf{0.0021}) & 1.50 \\
\bottomrule[1pt]
\end{tabular}
\end{table}

\subsection{Experiments on Regression and Binary Classification Benchmark Datasets}
\label{sec:4.2}

Below, we conduct experiments on regression benchmark datasets from UCI\footnote{\url{https://archive.ics.uci.edu/ml/datasets.php}} and binary classification benchmark datasets from LIBSVM\footnote{\url{https://www.csie.ntu.edu.tw/~cjlin/libsvmtools/datasets/}}.

We consider experimental settings with both synthetically created covariate shift and naturally occurring covariate shift. 
To perform train-test split for the datasets with naturally occurring covariate shift, we follow \citet{ahmed2014dataset}, \citet{pmlr-v51-chen16d}, and \citet{NIPS2006_a74c3bae} to separate the auto mpg dataset, the bike sharing dataset, the parkinsons dataset, and the wine quality dataset based on different origins, different semesters, different age ranges, and different types, respectively.
In the rest of the datasets, we synthetically introduce covariate shift in the following way similarly to \citet{cortes2008sample}.
First, we use Z-score normalization to preprocess all the input samples.
Then, an example $(\bx, y)$ is assigned to the training dataset with probability $\exp(v)/(1+\exp(v))$ and to the test dataset with probability $1/(1+\exp(v))$, where $v=16\bw^\top\bx/\sigma$, $\sigma$ is the standard deviation of $\bw^\top\bx$, and $\bw\in\bR^d$ is some given projection vector.
To ensure that the methods are tested in challenging covariate shift situations, we randomly sample projection directions and choose the one such that the classifier trained on the training dataset generalizes worst to the test dataset for train-test split.

By following the above procedure, we split the datasets into training datasets and test datasets (with some randomness in synthetic cases).
Then we sample a certain number (depending on the size of the dataset) of training samples and test input samples for training. We use the rest of test samples for evaluating the performance.
We run 100 trials for each dataset.\footnotemark
\footnotetext{It means that we conduct the experiment for each dataset 100 times with different random draws of training and test samples.}

The models and the hyperparameter tuning procedure follow what we discussed in Section~\ref{sec:4.1}.
To reduce computation time, we employ the two heuristics mentioned in Section~\ref{sec:4.1} for the one-step approach.
In addition, as discussed in Section~\ref{sec:3.1}, we use the squared loss as the surrogate loss function for all the methods including the one-step approach in the experiments.

The experimental results on benchmark datasets are summarized in Table~\ref{tab:benchmark}.
The table shows the proposed one-step approach outperforms or is comparable to the baseline methods with the best performance, which suggests that it is a promising method for covariate shift adaptation.

\begin{table}[t]
\centering
\caption{Mean squared test errors/mean test misclassification rates averaged over 100 trials on regression/binary classification benchmark datasets.
The numbers in the brackets are the standard deviations.
All the error values are normalized so that the mean error by ``ERM'' will be one.
For each dataset, the best method and comparable ones based on the \emph{paired t-test} at the significance level 5\% are described in bold face.
The upper half are regression datasets and the lower half are binary classification datasets.}
\vspace{2mm}
\label{tab:benchmark}
\resizebox{\textwidth}{!}{
\begin{tabular}{c|cccccccc}
\toprule[1pt]
Dataset & ERM & \tabincell{c}{ERM\\(median)} & EIWERM & \tabincell{c}{EIWERM\\(median)} & RIWERM & \tabincell{c}{RIWERM\\(median)} & \tabincell{c}{one-step\\(uLSIF)} & \tabincell{c}{one-step\\(median)} \\
\midrule
auto & \tabincell{c}{1.00\\(0.22)} & \tabincell{c}{1.22\\(0.29)} & \tabincell{c}{1.09\\(0.25)} & \tabincell{c}{1.08\\(0.25)} & \tabincell{c}{1.09\\(0.26)} & \tabincell{c}{1.08\\(0.23)} & \textbf{\tabincell{c}{1.00\\(0.26)}} & \textbf{\tabincell{c}{0.99\\(0.21)}} \\
bike & \tabincell{c}{1.00\\(0.10)} & \tabincell{c}{0.97\\(0.10)} & \tabincell{c}{1.05\\(0.19)} & \tabincell{c}{0.97\\(0.10)} & \tabincell{c}{1.02\\(0.13)} & \tabincell{c}{0.98\\(0.10)} & \tabincell{c}{1.05\\(0.15)} & \textbf{\tabincell{c}{0.95\\(0.08)}} \\
parkinsons & \tabincell{c}{1.00\\(0.28)} & \tabincell{c}{0.94\\(0.22)} & \tabincell{c}{1.05\\(0.37)} & \tabincell{c}{0.93\\(0.17)} & \tabincell{c}{1.02\\(0.37)} & \tabincell{c}{0.92\\(0.16)} & \tabincell{c}{0.78\\(0.07)} & \textbf{\tabincell{c}{0.76\\(0.05)}} \\
wine & \tabincell{c}{1.00\\(0.22)} & \tabincell{c}{0.95\\(0.13)} & \tabincell{c}{1.05\\(0.24)} & \tabincell{c}{0.95\\(0.12)} & \tabincell{c}{1.07\\(0.36)} & \tabincell{c}{0.95\\(0.14)} & \tabincell{c}{0.96\\(0.10)} & \textbf{\tabincell{c}{0.90\\(0.07)}} \\
\midrule
australian & \tabincell{c}{32.02\\(16.88)} & \tabincell{c}{31.62\\(17.88)} & \tabincell{c}{30.70\\(16.35)} & \tabincell{c}{31.00\\(17.47)} & \tabincell{c}{29.82\\(14.83)} & \tabincell{c}{31.81\\(17.52)} & \tabincell{c}{28.45\\(13.86)} & \textbf{\tabincell{c}{25.57\\(12.74)}} \\
breast & \tabincell{c}{22.72\\(13.12)} & \tabincell{c}{22.13\\(10.36)} & \tabincell{c}{21.84\\(13.16)} & \tabincell{c}{21.82\\(11.20)} & \tabincell{c}{23.58\\(12.57)} & \tabincell{c}{22.00\\(13.38)} & \textbf{\tabincell{c}{16.55\\(9.09)}} & \tabincell{c}{22.51\\(12.56)} \\
diabetes & \tabincell{c}{45.78\\(8.88)} & \tabincell{c}{43.35\\(9.56)} & \tabincell{c}{42.44\\(7.65)} & \tabincell{c}{41.67\\(8.66)} & \tabincell{c}{43.26\\(8.42)} & \tabincell{c}{44.65\\(10.07)} & \textbf{\tabincell{c}{39.58\\(5.29)}} & \textbf{\tabincell{c}{38.57\\(6.36)}} \\
heart & \tabincell{c}{36.39\\(11.90)} & \tabincell{c}{34.91\\(12.45)} & \textbf{\tabincell{c}{32.06\\(11.05)}} & \tabincell{c}{35.13\\(12.57)} & \textbf{\tabincell{c}{33.39\\(12.24)}} & \tabincell{c}{35.77\\(15.26)} & \textbf{\tabincell{c}{31.39\\(10.36)}} & \textbf{\tabincell{c}{31.88\\(11.95)}} \\
sonar & \tabincell{c}{39.57\\(7.10)} & \textbf{\tabincell{c}{39.03\\(6.69)}} & \tabincell{c}{39.19\\(7.00)} & \textbf{\tabincell{c}{38.77\\(6.37)}} & \textbf{\tabincell{c}{38.83\\(7.15)}} & \textbf{\tabincell{c}{39.03\\(6.39)}} & \textbf{\tabincell{c}{37.69\\(7.17)}} & \tabincell{c}{39.23\\(7.17)} \\
\bottomrule[1pt]
\end{tabular}}
\end{table}

\section{Extension to Multi-class Classification with Neural Networks}
\label{sec:5}
In this section, we extend the proposed method to multi-class classification and conduct experiments with neural networks.

Consider a $K$-class classification problem with an input space $\cX\subset\bR^d$ and an output space $\cY=\{1,\ldots,K\}$, and let $\bof\colon\cX\rightarrow\bR^K$ be the classifier to be trained for this problem and $\ell\colon \bR^K\times\cY \to \bR_+$ be the loss function in the test risk $R(\bof)$ (cf. Eq.~\eqref{eq: risk}). 
The zero-one loss $\ell(\bo, y)=I(y\neq\argmax_{k\in\cY} \bo_k)$ is a typical choice, where $\bo\in\bR^K$ and $\bo_k$ denotes the $k$-th element of $\bo$. 
As discussed in Section~\ref{sec:3.1}, we can use a surrogate loss $\ell_\mathrm{UB}$ that bounds the zero-one loss from above, e.g., the \emph{softmax cross-entropy loss}, to obtain a tractable upper-bound $J_\mathrm{UB}(\bof, g)$ (cf. Eq.~\eqref{eq: upper bound2}). 
We present a gradient-based alternating minimization algorithm described in Algorithm~\ref{alg:gradient} that is more convenient for training neural networks.

\begin{algorithm}[ht]
\caption{Gradient-based Alternating Minimization}
\label{alg:gradient}
\begin{algorithmic}[1]
\linespread{1.1}\selectfont
\STATE $\cZ^{\tr}, \cX^{\te} \gets \left\{\left(\bx_i^\tr, y_i^\tr\right)\right\}_{i=1}^{n_\tr}, \left\{\bx_i^\te\right\}_{i=1}^{n_\te}$
\STATE $\cA\gets$ a gradient-based optimizer
\STATE $\bof\gets$ an arbitrary classifier
\FOR{$\mathrm{round}=0, 1, \ldots, \mathrm{numOfRounds}-1$}
\FOR{$\mathrm{epoch}=0, 1, \ldots, \mathrm{numOfEpochsForG}-1$}
\FOR{$i=0, 1, \ldots, \mathrm{numOfMiniBatches}-1$}
\STATE $\cZ_i^{\tr}, \cX_i^{\te}\gets\mathrm{sampleMiniBatch}(\cZ^{\tr}, \cX^{\te})$
\STATE $g\gets\cA(g, \nabla_g\widehat{J}_\mathrm{UB}(\bof, g; \cZ_i^{\tr}\cup\cX_i^{\te}))$
\ENDFOR
\ENDFOR
\FOR{$\mathrm{epoch}=0, 1, \ldots, \mathrm{numOfEpochsForF}-1$}
\FOR{$i=0, 1, \ldots, \mathrm{numOfMiniBatches}-1$}
\STATE $\cZ_i^{\tr}\gets\mathrm{sampleMiniBatch}(\cZ^{\tr})$
\STATE $w_j\gets\max(g(\bx_j), 0)$, $\forall(\bx_j, \cdot)\in\cZ_i^{\tr}$
\STATE $w_j\gets w_j/\sum_j w_j$, $\forall j$
\STATE $L_i\gets\sum_{(\bx_j, y_j)\in\cZ_i^{\tr}}w_j\ell_\mathrm{UB}(\bof(\bx_j), y_j)$
\STATE $\bof\gets\cA(\bof, \nabla_{\bof}L_i)$
\ENDFOR
\ENDFOR
\ENDFOR
\end{algorithmic}
\end{algorithm}

Below, we design a covariate shift setting and conduct experiments on the  Fashion-MNIST \citep{xiao2017fashion} and Kuzushiji-MNIST \citep{clanuwat2018deep} benchmark datasets for image classification using convolutional neural networks (CNNs).

Based on the fact that the labels of the images from those datasets are invariant to rotation transformation, we introduce covariate shift to the image datasets in the following way: we rotate each image $I_i$ in the training sets by angle $\theta_i$, where $\theta_i/180^\circ$ is drawn from a Beta distribution $\mathrm{Beta}(a, b)$, and rotate each image $J_i$ in the test sets by angle $\phi_i$,  where $\phi_i/180^\circ$ is drawn from another Beta distribution $\mathrm{Beta}(b, a)$. 
The parameters $a$ and $b$ control the shift level, and we test three different levels in our experiments: $(a, b) = (2, 4)$, $(2, 5)$, and $(2, 6)$. 
Since our experiments are conducted in an inductive manner, we also rotate each image $I_i$ in the training sets by angle $\psi_i$, where $\psi_i/180^\circ$ is drawn from the Beta distribution $\mathrm{Beta}(b, a)$ to obtain the unlabeled test images for training.

We compare our one-step approach with three baseline methods: the ordinary ERM, EIWERM with $\gamma=0.5$, and RIWERM with $\alpha=0.5$. 
We use the softmax cross-entropy loss as a surrogate loss and use 5-layer CNNs, which consist of 2 convolutional layers with pooling and 3 fully connected layers, to model the classifier $\bof$ and the weight model $g$. 
In order to learn useful weights, we pretrain $g$ in a binary classification problem whose goal is to discriminate between $\left\{\bx_i^\tr\right\}_{i=1}^{n_\tr}$ and $\left\{\bx_i^\te\right\}_{i=1}^{n_\te}$ and freeze the parameters in the first two convolutional layers. 
We train $\bof$ and $g$ for 20 rounds for the one-step approach, where a round consists of 5 epochs of training $g$ followed by 10 epochs of training $\bof$: we train the models for 300 epochs in total. 
We use stochastic gradient descent with a learning rate of $10^{-4}$ and mini-batch size of 128 for training $g$ and use Adam \citep{kingma:adam} with an initial learning rate of $10^{-3}$ halved every 4 rounds and mini-batch size of 128 for training $\bof$. 
For a fair comparison, we train the (relative) importance models for 100 epochs and the classifiers for 200 epochs in the baseline methods.

The experimental results summarized in Table~\ref{tab:deep} verify the effectiveness of our one-step approach in image classification problems with neural networks. 
Specifically, the table shows that the ordinary ERM performs poorly under covariate shift, the weighted methods all improve the performance, and the one-step approach further improves the performance especially under large covariate shift (i.e., the difference between shift parameters $a$ and $b$ is large).

\begin{table}[t]
\centering
\caption{Mean test classification accuracy averaged over 5 trials on image datasets with neural networks.
The numbers in the brackets are the standard deviations.
For each dataset, the best method and comparable ones based on the \emph{paired t-test} at the significance level 5\% are described in bold face.}
\vspace{2mm}
\label{tab:deep}
\resizebox{\textwidth}{!}{
\begin{tabular}{c|c|cccc}
\toprule[1pt]
Dataset & \tabincell{c}{Shift Level\\($a$, $b$)} & ERM & EIWERM & RIWERM & one-step\\
\midrule
 & (2, 4) & 81.71(0.17) & 84.02(0.18) & 84.12(0.06) & \textbf{85.07(0.08)} \\
Fashion-MNIST & (2, 5) & 72.52(0.54) & 76.68(0.27) & 77.43(0.29) & \textbf{78.83(0.20)} \\
 & (2, 6) & 60.10(0.34) & 65.73(0.34) & 66.73(0.55) & \textbf{69.23(0.25)} \\
\midrule
 & (2, 4) & 77.09(0.18) & 80.92(0.32) & 81.17(0.24) & \textbf{82.45(0.12)} \\
Kuzushiji-MNIST & (2, 5) & 65.06(0.26) & 71.02(0.50) & 72.16(0.19) & \textbf{74.03(0.16)} \\
 & (2, 6) & 51.24(0.30) & 58.78(0.38) & 60.14(0.93) & \textbf{62.70(0.55)} \\
\bottomrule[1pt]
\end{tabular}}
\end{table}

\section{Conclusion}
\label{sec:6}

In this work, we studied the problem of covariate shift adaptation.
Unlike the dominating two-step approaches in the literature, we proposed a one-step approach that learns the predictive model and the associated weights simultaneously by following Vapnik’s principle.
Our experiments highlighted the advantage of our method over previous two-step approaches,
suggesting that the proposed one-step approach is a promising method for covariate shift adaptation.
For future work, it would be interesting to investigate whether the proposed method is still useful in covariate shift adaptation for nowadays prevalent extremely deep neural networks because importance weighting is valid only when the model used for learning is misspecified \citep{sugiyama2012machine}.

\section*{Acknowledgments}

NL was supported by MEXT scholarship No.\ 171536. IY and MS were supported by JST CREST Grant Number JPMJCR18A2. IY acknowledges the support of the ANR as part of the ``Investissements d’avenir'' program, reference ANR-19-P3IA-0001 (PRAIRIE 3IA Institute).

\bibliographystyle{spbasic}      
\bibliography{reference}   


\appendix

\section{Proof of Lemma~\ref{lemma: uniform bound}}
\label{sec: proof of uniform bound lemma}
\begin{proof}
Let $\Phi(S)=\sup_{f\in\cF, g\in\cG_+}\left(J_\mathrm{UB}(f, g)-\widehat{J}_\mathrm{UB}(f, g; S)\right)$ and $S'$ be a set differing from $S$ on exactly one sample point.
Then, since the difference of suprema does not exceed the supremum of the difference, we have
\[\Phi(S')-\Phi(S)\leq\sup_{f\in\cF, g\in\cG_+}\left(\widehat{J}_\mathrm{UB}(f, g; S)-\widehat{J}_\mathrm{UB}(f, g; S')\right).\]
If the differing sample point is a training sample, then \[\Phi(S')-\Phi(S) \leq 2MG\cdot\frac{2}{n_\tr}MG + M^2\cdot\frac{1}{n_\tr}G^2 = \frac{5}{n_\tr}M^2G^2.\]
On the other hand, if the differing sample point is a test sample, then 
\[\Phi(S')-\Phi(S) \leq M^2\cdot\frac{2}{n_\te}\cdot2G \leq \frac{5}{n_\te}M^2G^2.\]
Similarly, we can obtain the same result for bounding $\Phi(S)-\Phi(S')$.
Then, by McDiarmid's inequality, for any $\delta>0$, with probability at least $1-\delta$, the following holds:
\begin{equation*}
    \Phi(S)
    \leq \bE_S[\Phi(S)] + 5M^2G^2\sqrt{\frac{\log\frac{1}{\delta}}{2}}\left(\frac{1}{\sqrt{n_\tr}}+\frac{1}{\sqrt{n_\te}}\right).
\end{equation*}

Let $S_\tr=\left\{\left(\bx_i^\tr, y_i^\tr\right)\right\}_{i=1}^{n_\tr}$ and $S_\te=\left\{\bx_i^\te\right\}_{i=1}^{n_\te}$.
We next bound the expectation in the right-hand side:
\begin{align*}
    \bE_S[\Phi(S)]&=\bE_S\left[\sup_{f\in\cF, g\in\cG_+}\left(J_\mathrm{UB}(f, g)-\widehat{J}_\mathrm{UB}(f, g; S)\right)\right]
    \leq ({\rm \uppercase\expandafter{\romannumeral1}}) + M^2({\rm \uppercase\expandafter{\romannumeral2}}) + 2   M^2({\rm \uppercase\expandafter{\romannumeral3}}),
\end{align*}
where 
\begin{align*}
    ({\rm \uppercase\expandafter{\romannumeral1}})
    &=\bE_{S_\tr}\Bigg[\sup_{f\in\cF, g\in\cG_+}\Bigg(\left(\bE_{(\bx^\tr, y^\tr)}\left[\ell_\mathrm{UB}(f(\bx^\tr), y^\tr)g(\bx^\tr)\right]\right)^2\\
    &\phantom{=\bE_{S_\tr}\Bigg[\sup_{f\in\cF, g\in\cG_+}\Bigg((}
    - \left(\frac{1}{n_\tr}\sum_{i=1}^{n_\tr}\ell_\mathrm{UB}(f(\bx^\tr_i), y^\tr_i)g(\bx^\tr_i)\right)^2\Bigg)\Bigg],\\
    ({\rm \uppercase\expandafter{\romannumeral2}})
    &=\bE_{S_\tr}\left[\sup_{g\in\cG_+}\left(\bE_{\bx^\tr}\left[g^2(\bx^\tr)\right] - \frac{1}{n_\tr}\sum_{i=1}^{n_\tr}g^2(\bx^\tr_i)\right)\right],\\
    ({\rm \uppercase\expandafter{\romannumeral3}})
    &=\bE_{S_\te}\left[\sup_{g\in\cG_+}\left(\frac{1}{n_\te}\sum_{i=1}^{n_\te}g(\bx^\te_i) - \bE_{\bx^\te\sim p_\te(\bx)}\left[g(\bx^\te)\right]\right)\right].
\end{align*}
Then we bound the above three terms as follows:
\begin{align*}
    ({\rm \uppercase\expandafter{\romannumeral1}})
    &\leq \bE_{S_\tr}\Bigg[\sup_{f\in\cF, g\in\cG_+}\Bigg(\bE_{\tilde{S}_\tr}\left(\frac{1}{n_\tr}\sum_{i=1}^{n_\tr}\ell_\mathrm{UB}(f(\tilde{\bx}^\tr_i), \tilde{y}^\tr_i)g(\tilde{\bx}^\tr_i)\right)^2\\
    &\phantom{\leq \bE_{S_\tr}\Bigg[\sup_{f\in\cF, g\in\cG_+}\Bigg(}
    - \left(\frac{1}{n_\tr}\sum_{i=1}^{n_\tr}\ell_\mathrm{UB}(f(\bx^\tr_i), y^\tr_i)g(\bx^\tr_i)\right)^2\Bigg)\Bigg]\\
    \intertext{\qquad\ (points in $\tilde{S}_\tr$ are sampled in an i.i.d.\@ fashion from $p_\tr(\bx, y)$)}
    &\leq \bE_{S_\tr, \tilde{S}_\tr}\Bigg[\sup_{f\in\cF, g\in\cG_+}\Bigg(\left(\frac{1}{n_\tr}\sum_{i=1}^{n_\tr}\ell_\mathrm{UB}(f(\tilde{\bx}^\tr_i), \tilde{y}^\tr_i)g(\tilde{\bx}^\tr_i)\right)^2\\
    &\phantom{\leq \bE_{S_\tr, \tilde{S}_\tr}\Bigg[\sup_{f\in\cF, g\in\cG_+}\Bigg(}
    - \left(\frac{1}{n_\tr}\sum_{i=1}^{n_\tr}\ell_\mathrm{UB}(f(\bx^\tr_i), y^\tr_i)g(\bx^\tr_i)\right)^2\Bigg)\Bigg]\\
    &\leq \bE_{S_\tr, \tilde{S}_\tr}\Bigg[\sup_{f\in\cF, g\in\cG_+}\frac{2MG}{n_\tr}\Bigg|\sum_{i=1}^{n_\tr}\big(\ell_\mathrm{UB}(f(\tilde{\bx}^\tr_i), \tilde{y}^\tr_i)g(\tilde{\bx}^\tr_i)\\
    &\phantom{\leq \bE_{S_\tr, \tilde{S}_\tr}\Bigg[\sup_{f\in\cF, g\in\cG_+}\frac{2MG}{n_\tr}\Bigg|\sum_{i=1}^{n_\tr}\big(}
    -\ell_\mathrm{UB}(f(\bx^\tr_i), y^\tr_i)g(\bx^\tr_i)\big)\Bigg|\Bigg]\\
    &= \bE_{\bsigma, S_\tr, \tilde{S}_\tr}\Bigg[\sup_{f\in\cF, g\in\cG_+}\frac{2MG}{n_\tr}\Bigg|\sum_{i=1}^{n_\tr}\sigma_i\big(\ell_\mathrm{UB}(f(\tilde{\bx}^\tr_i), \tilde{y}^\tr_i)g(\tilde{\bx}^\tr_i)\\
    &\phantom{= \bE_{\bsigma, S_\tr, \tilde{S}_\tr}\Bigg[\sup_{f\in\cF, g\in\cG_+}\frac{2MG}{n_\tr}\Bigg|\sum_{i=1}^{n_\tr}\sigma_i\big(}
    -\ell_\mathrm{UB}(f(\bx^\tr_i), y^\tr_i)g(\bx^\tr_i)\big)\Bigg|\Bigg]\\
    \intertext{\qquad\ ($\{\sigma_i\}_{i=1}^{n_\tr}$ is a Rademacher sequence)}
    &\leq 4MG\bE_{\bsigma, S_\tr}\left[\sup_{f\in\cF, g\in\cG_+}\left|\frac{1}{n_\tr}\sum_{i=1}^{n_\tr}\sigma_i\ell_\mathrm{UB}(f(\bx^\tr_i), y^\tr_i)g(\bx^\tr_i)\right|\right]\\
    &\leq 4MG\bE_{\bsigma, S_\tr}\left[\sup_{f\in\cF, g\in\cG}\frac{1}{n_\tr}\sum_{i=1}^{n_\tr}\sigma_i\ell_\mathrm{UB}(f(\bx^\tr_i), y^\tr_i)g(\bx^\tr_i)\right]\\
    &\leq 2MG\Bigg(\bE_{\bsigma, S_\tr}\left[\sup_{f\in\cF, g\in\cG}\frac{1}{n_\tr}\sum_{i=1}^{n_\tr}\sigma_i\left(\ell_\mathrm{UB}(f(\bx^\tr_i), y^\tr_i)+g(\bx^\tr_i)\right)^2\right]\\
    &\qquad + \bE_{\bsigma, S_\tr}\left[\sup_{f\in\cF}\frac{1}{n_\tr}\sum_{i=1}^{n_\tr}\sigma_i\ell_\mathrm{UB}^2(f(\bx^\tr_i), y^\tr_i)\right]
    + \bE_{\bsigma, S_\tr}\left[\sup_{g\in\cG}\frac{1}{n_\tr}\sum_{i=1}^{n_\tr}\sigma_i g^2(\bx^\tr_i)\right]\Bigg)\\
    &\leq 2MG\Bigg(2\left(M+G\right)\bE_{\bsigma, S_\tr}\left[\sup_{f\in\cF, g\in\cG}\frac{1}{n_\tr}\sum_{i=1}^{n_\tr}\sigma_i\left(\ell_\mathrm{UB}(f(\bx^\tr_i), y^\tr_i)+g(\bx^\tr_i)\right)\right]\\
    &\qquad + 2M\bE_{\bsigma, S_\tr}\left[\sup_{f\in\cF}\frac{1}{n_\tr}\sum_{i=1}^{n_\tr}\sigma_i\ell_\mathrm{UB}(f(\bx^\tr_i), y^\tr_i)\right]
    + 2G\bE_{\bsigma, S_\tr}\left[\sup_{g\in\cG}\frac{1}{n_\tr}\sum_{i=1}^{n_\tr}\sigma_i g(\bx^\tr_i)\right]\Bigg)\\
    &\quad\ \text{(Ledoux-Talagrand contraction lemma \citep{ledoux2013probability})}\\
    &\leq 4MG\left(2M+G\right)L\fR_{n_\tr}(\cF)+4MG\left(M+2G\right)\fR_{n_\tr}(\cG),\\
    &\quad\ \text{(Ledoux-Talagrand contraction lemma)}
\end{align*}
\begin{align*}
    ({\rm \uppercase\expandafter{\romannumeral2}})&=\bE_{S_\tr}\left[\sup_{g\in\cG_+}\left(\bE_{\tilde{S}_\tr}\left[\frac{1}{n_\tr}\sum_{i=1}^{n_\tr}g^2(\tilde{\bx}^\tr_i)\right] - \frac{1}{n_\tr}\sum_{i=1}^{n_\tr}g^2(\bx^\tr_i)\right)\right]\\
    &\leq \bE_{S_\tr, \tilde{S}_\tr}\left[\sup_{g\in\cG_+}\frac{1}{n_\tr}\sum_{i=1}^{n_\tr}\left(g^2(\tilde{\bx}^\tr_i)-g^2(\bx^\tr_i)\right)\right]\\
    &= \bE_{\bsigma, S_\tr, \tilde{S}_\tr}\left[\sup_{g\in\cG_+}\frac{1}{n_\tr}\sum_{i=1}^{n_\tr}\sigma_i\left(g^2(\tilde{\bx}^\tr_i)-g^2(\bx^\tr_i)\right)\right]\\
    &\leq 2\bE_{\bsigma, S_\tr}\left[\sup_{g\in\cG_+}\frac{1}{n_\tr}\sum_{i=1}^{n_\tr}\sigma_ig^2(\bx^\tr_i)\right]\\
    &\leq 4G\fR_{n_\tr}(\cG),\quad\text{(Ledoux-Talagrand contraction lemma)}
\end{align*}
\begin{align*}
    ({\rm \uppercase\expandafter{\romannumeral3}})&=\bE_{S_\te}\left[\sup_{g\in\cG_+}\left(\frac{1}{n_\te}\sum_{i=1}^{n_\te}g(\bx^\te_i) - \bE_{\tilde{S}_\te}\left[\frac{1}{n_\te}\sum_{i=1}^{n_\te}g(\tilde{\bx}^\te_i)\right]\right)\right]\\
    &\leq \bE_{S_\te, \tilde{S}_\te}\left[\sup_{g\in\cG_+}\frac{1}{n_\te}\sum_{i=1}^{n_\te}\left(g(\bx^\te_i)-g(\tilde{\bx}^\te_i)\right)\right]\\
    &= \bE_{\bsigma, S_\te, \tilde{S}_\te}\left[\sup_{g\in\cG_+}\frac{1}{n_\te}\sum_{i=1}^{n_\te}\sigma_i\left(g(\bx^\te_i)-g(\tilde{\bx}^\te_i)\right)\right]\\
    &\leq 2\bE_{\bsigma, S_\te}\left[\sup_{g\in\cG_+}\frac{1}{n_\te}\sum_{i=1}^{n_\te}\sigma_i g(\bx^\te_i)\right]
    \leq 2\fR_{n_\te}(\cG).
\end{align*}
By summarizing all the results above, we complete the proof.\footnote{In fact, the bound presented in Lemma~\ref{lemma: uniform bound} is looser than the result that we obtained here. We did this for saving the space and making the bound more readable.}
\end{proof}

\section{Proof of Theorem~\ref{thm: generalization error bound}}
\label{sec: proof of generalization error bound theorem}
\begin{proof}
Let $(f^*_J, g^*_J) = \argmin_{(f,g)\in\cF\times\cG}J_\mathrm{UB}(f, g)$.
Then for any $\delta>0$, by McDiarmid's inequality, with probability at least $1-\delta$, we have
\begin{equation*}
    \widehat{J}_\mathrm{UB}(f^*_J, g^*_J; S)\leq \bE_S[\widehat{J}_\mathrm{UB}(f^*_J, g^*_J; S)] + 5M^2G^2\sqrt{\frac{\log\frac{1}{\delta}}{2}}\left(\frac{1}{\sqrt{n_\tr}}+\frac{1}{\sqrt{n_\te}}\right).
\end{equation*}
Since $\bE\left[X^2\right]=\left(\bE[X]\right)^2+{\rm Var}[X]$, we have
\begin{align*}
    \bE_S[\widehat{J}_\mathrm{UB}(f^*_J, g^*_J; S)] &= J_\mathrm{UB}(f^*_J, g^*_J) + \frac{1}{n_\tr}{\rm Var}\left[\ell_\mathrm{UB}(f(\tilde{\bx}^\tr_1), \tilde{y}^\tr_1)g(\tilde{\bx}^\tr_1)\right]\\
    &\leq J_\mathrm{UB}(f^*_J, g^*_J) + \frac{1}{n_\tr}M^2G^2,
\end{align*}
and thus,
\begin{equation*}
    \widehat{J}_\mathrm{UB}(f^*_J, g^*_J; S)\leq J_\mathrm{UB}(f^*_J, g^*_J) + 5M^2G^2\sqrt{\frac{\log\frac{1}{\delta}}{2}}\left(\frac{1}{\sqrt{n_\tr}}+\frac{1}{\sqrt{n_\te}}\right) + M^2G^2\frac{1}{n_\tr}.
\end{equation*}
Therefore, according to \eqref{eq: upper bound1}, \eqref{eq: upper bound2} and \eqref{eq: uniform bound}, we have
\begin{flalign*}
    &\quad\ \frac{1}{2}R^2(\widehat{f}) - J_\mathrm{UB}(f^*_J, g^*_J)\\
    &\leq \left(J_\mathrm{UB}(\widehat{f}, \widehat{g}) - \widehat{J}_\mathrm{UB}(\widehat{f}, \widehat{g}; S)\right) + \left(\widehat{J}_\mathrm{UB}(\widehat{f}, \widehat{g}; S) - \widehat{J}_\mathrm{UB}(f^*_J, g^*_J; S)\right)\\
    &\phantom{\leq\ }+ \left(\widehat{J}_\mathrm{UB}(f^*_J, g^*_J; S) - J_\mathrm{UB}(f^*_J, g^*_J)\right)\\
    &\leq \Bigg(8MG\left(M+G\right)\left(L\fR^\tr_{n_\tr}(\cF)+\fR^\tr_{n_\tr}(\cG)\right) + 4M^2\fR^\te_{n_\te}(\cG)\\
    &\phantom{\leq \Bigg(}+ 5M^2G^2\sqrt{\frac{\log\frac{1}{\delta}}{2}}\left(\frac{1}{\sqrt{n_\tr}}+\frac{1}{\sqrt{n_\te}}\right)\Bigg)+0\\
    &\phantom{\leq\ }+ \Bigg(5M^2G^2\sqrt{\frac{\log\frac{1}{\delta}}{2}}\left(\frac{1}{\sqrt{n_\tr}}+\frac{1}{\sqrt{n_\te}}\right) + M^2G^2\frac{1}{n_\tr}\Bigg).
\end{flalign*}
Rearranging the equation above, we obtain Eq.~\eqref{eq: gen error bound}.
\end{proof}

\end{document}